\documentclass[sn-mathphys,Numbered]{sn-jnl}

\usepackage{bm}
\usepackage{graphicx}%
\usepackage{multirow}%
\usepackage{amsmath,amssymb,amsfonts}%
\usepackage{amsthm}%
\usepackage{mathrsfs}%
\usepackage[title]{appendix}%
\usepackage{xcolor}%
\usepackage{textcomp}%
\usepackage{ascmac}
\usepackage{manyfoot}%
\usepackage{booktabs}%
\usepackage{algorithm}%
\usepackage{algorithmicx}%
\usepackage{algpseudocode}%
\usepackage{listings}%
\usepackage{amsmath,amsfonts}
\usepackage{array}

\newtheorem{theorem}{Theorem}

\begin{document}
\title[Asymmetric Design of Control Barrier Function for Multiagent Autonomous Robotic Systems]{Asymmetric Design of Control Barrier Function for Multiagent Autonomous Robotic Systems}


\author*[1]{\fnm{Hiroki} \sur{Etchu}}\email{etchu.h.aa@m.titech.ac.jp}

\author[2]{\fnm{Yuki} \sur{Origane}}\email{origane.y.aa@m.titech.ac.jp}
\equalcont{These authors contributed equally to this work.}

\author[3]{\fnm{Daisuke} \sur{Kurabayashi}}\email{dkura@irs.ctrl.titech.ac.jp}
\equalcont{These authors contributed equally to this work.}

\affil*[1]{\orgdiv{Systems and control engineering}, \orgname{Tokyo Institute of Technology}, \orgaddress{\street{Ookayama 2-12-1}, \city{Meguro}, \postcode{152-8550}, \state{Tokyo}, \country{Japan}}}

\affil[2]{\orgdiv{Systems and control engineering}, \orgname{Tokyo Institute of Technology}, \orgaddress{\street{Ookayama 2-12-1}, \city{Meguro}, \postcode{152-8550}, \state{Tokyo}, \country{Japan}}}

\affil[3]{\orgdiv{Systems and control engineering}, \orgname{Tokyo Institute of Technology}, \orgaddress{\street{Ookayama 2-12-1}, \city{Meguro}, \postcode{152-8550}, \state{Tokyo}, \country{Japan}}}


\abstract{
In this paper, we propose a method to avoid “no-solution” situations of the control barrier function (CBF) for distributed collision avoidance in a multiagent autonomous robotic system (MARS). MARS, which is composed of distributed autonomous mobile robots, is expected to effectively perform cooperative tasks such as searching in a certain area. Therefore, collision avoidance must be considered when implementing MARS in the real world. The CBF is effective for solving collision-avoidance problems. However, in extreme conditions where many robots congregate at one location, the CBF constraints that ensure a safe distance between robots may be violated. We theoretically demonstrate that this problem can occur in certain situations, and introduce an asymmetric design for the inequality constraints of CBF.  We asymmetrically decentralized inequality constraints with weight functions using the absolute speed of the robot so that other robots can take over the constraints of the robot in severe condition. We demonstrate the effectiveness of the proposed method in a two-dimensional situation wherein multiple robots congregate at one location. We implement the proposed method on real robots and the confirmed the effectiveness of this theory.
}




\maketitle
\section{Introduction}\label{sec1}
A significant amount of research has been conducted on the multiagent autonomous robotic system (MARS) \cite{c1}-\cite{c11}, which is superior to individual robots in terms of the system flexibility, scalability, and fault tolerance. MARS, which is composed of autonomous mobile robots, is expected to effectively perform cooperative tasks such as searching in a certain area. In such tasks, situations may arise wherein multiple robots congregate at one location. Because a real robot has a physical body with a nonzero volume, collision avoidance must be considered in these scenarios. The control barrier function (CBF) \cite{c12} is effective for solving collision-avoidance problems with a low computational cost. Ames \textit{et al}. \cite{c13} applied the CBF to mobile and legged robots. CBF was also applied to multiagent robot systems in some research articles \cite{c14}-\cite{c19}. Wang \textit{et al}. \cite{c14} proposed and demonstrated a realistic solution for a multiple mobile robot system capable of escaping deadlocked situations. However, in extreme conditions, where many robots congregate at one location, the CBF constraints that ensure a safe distance between robots may be violated. This may lead to “no-solution” situations. In this study, we focus on avoiding the violations of the safety functions of CBF, because the constraints must be satisfied in all situations to guarantee safety from collisions.
 
We found that this problem of no-solution was caused by the decentralization of the CBF constraints to extend CBF to MARS. Most of the previous studies extending CBF to MARS employed a symmetric distribution of the inequality constraints among the robots. However, because of this symmetric distribution, robots often failed to find a common solution that satisfied all the distributed constraints. Therefore, the constraints must be decentralized, so that the robots will always have a common solution.

In this paper, we propose an asymmetric weight design of the constraints of CBF, to guarantee a common solution for a swarm of robots. First, we mathematically examine how a symmetric distribution violates the constraints when multiple robots congregate at one location. We then propose an asymmetric weight design for the decentralization of CBF. Through two-dimensional (2D) simulations where robots congregate at one location, we demonstrate the effectiveness of the proposed method. We confirm the effectiveness of a real environment through experiment which multiple robots congregate at one location. 

The remainder of this manuscript is organized as follows. The problem is formulated in Section 2. Section 3 shows how the constraints are violated, using a simple mathematical model. Section 4 introduces the asymmetric weight design for the distribution of constraints. In Section 5, we demonstrate the effectiveness of the proposed method through simulations. Section 6 demonstrates the effectiveness of the proposed method through an experiment with real robots. Section 7 concludes the paper and points out future research directions.
\section{Problem Statements}\label{sec2}
\subsection{Autonomous Robots}\label{subsec1}
We describe the essential conditions in this research. We consider a situation with three or more robots. Each robot satisfies (1) for a second-order differential system:
\begin{equation}
\begin{split}
m\ddot{\bm{x}}_i &=\bm{u}_i\\
\bm{u}_i &=k(\bm{x}_{\mathrm{goal}}-\bm{x}_i)-c\dot{\bm{x}}_i
\end{split}
\end{equation}
Where, $m$ is the mass of the robot, $k$ is the position feedback coefficient, and $c$ is the damper coefficient.
$\bm{x}_i$  is the vector of the robot$_i$'s position and $\bm{u}_i$ is the input to robot$_i$. Each robot can determine its relative position, relative velocity, and absolute velocity in relation to the other robots. We assume that a robot in MARS can obtain this information because several studies have used this for collision avoidance in automated driving \cite{c20}. However, they cannot identify others or know their inputs.
\subsection{Control Barrier Function (CBF)}
We employ the zeroing control barrier function (ZCBF) \cite{c13} in the framework for collision avoidance. Let state $\bm{q}\in \mathbb{R}^n$ follow (2) for the following input-affine system  with input $\bm{u}\in \mathbb{R}^m$.
\begin{equation}
\dot{\bm{q}} = f(\bm{q})+g(\bm{q})\bm{u}
\end{equation}
Where, $f$ and $g$ are local Lipschitz functions. $h(\bm{q}): \mathbb{R}^n  \rightarrow \mathbb{R}$ is a continuously differentiable function and $C = \{\bm{q}\in \mathbb{R}^n | h(\bm{q}) \geq 0\}$. Then, $h(\bm{q})$ is the ZCBF for set $C$.
Let us assume that, $a,b\geq0$. If the continuous function
$\alpha:(-b,a) \rightarrow (-\infty,\infty)$ is strictly monotonically increasing and $\alpha(0)=0$, $\alpha$ will be an extended class $K$ function. $C$ is a forward invariant set if there exists an extended class $K$ function $\alpha: \mathbb{R} \rightarrow \mathbb{R}$ that satisfies (3).
\begin{equation}
\underset {\bm{u}\in \mathbb{R}^m}{\mathrm{sup}}{[\dot{h}(\bm{q})+\alpha(h(\bm{q}))] \geq 0}
\end{equation}
Subsequently, if the function $h(\bm{q})$ is a ZCBF of $C$ and the Lipschitz continuous input $\bm{u}$ always satisfies (4), then $C$ will be a forward invariant set.
\begin{equation}
L_f h(\bm{q}) + L_g h(\bm{q})\bm{u} + \alpha(h(\bm{q})) \geq 0
\end{equation}
Where, $L_f$ and $L_g$ are the Lie derivatives along $f$ and $g$ respectively. We must select the control input $\bm{u}$ that satisfies (4) to guarantee $\bm{q}(0)\in C \rightarrow \bm{q}(t) \in C \ \forall t\geq 0$.
We must discuss how to find an input that satisfies the inequality constraints because an input for a robot in MARS along with its target behavior may not always satisfy the inequality constraints. Let $\bm{\hat{u}}$ be the control input that does not consider collision avoidance. We modify the control input to achieve the control target satisfying (4) by successively solving the minimization problem for $\bm{u}$, which minimizes the difference between $\bm{u}$ and $\bm{\hat{u}}$.
The minimization problem can be formulated as a quadratic program (QP), as shown in (5).
\begin{equation}
\begin{gathered}
\underset {\bm{u}\in \mathbb{R}^m}{\mathrm{argmin}}{\frac{1}{2}||\bm{\hat{u}}-\bm{u}||^2}\\
L_f h(\bm{q}) + L_g h(\bm{q})\bm{u} + \alpha(h(\bm{q})) \geq 0
\end{gathered}
\end{equation}
An example of a solution to (5) is (6), where the equality in (5) holds true.
\begin{equation}
\bm{u}=-\frac{L_f h(\bm{q})+\alpha(h(\bm{q}))}{||L_g h(\bm{q})||^2}L_g h(\bm{q})^\top
\end{equation}
Note that if the functions $L_f h(\bm{q})$,$L_g h(\bm{q})$,and $\alpha(h(\bm{q}))$ are Lipschitz continuous, the solution $\bm{u}$ of the QP will also be Lipschitz continuous.
\subsection{ZCBF for Distributed  Robots}
Herein, we introduce the conventional method of distributed ZCBF design for swarm robots \cite{c14,c19}. Let $\bm{q}_{ij}=[\bm{x}_{ij}^{\top},\bm{\dot{x}}_{ij}^{\top} ]^\top=[x_{ij},y_{ij},\dot{x}_{ij},\dot{y}_{ij} ]^\top=[(\bm{x}_j-\bm{x}_i)^\top,(\bm{\dot{x}}_j-\bm{\dot{x}}_i)^\top]^\top$, where $\bm{x}_i$ is the position of robot$_i$.
The ZCBF $h_0 (\bm{q}_{ij})$ between robots $i$ and $j$ can be expressed by (7).
\begin{equation}
\begin{gathered}
C_0= \{\bm{q}_{ij} \in \mathbb{R}^n | h_0 (\bm{q}_{ij} )\geq 0\} \\
h_0 (\bm{q}_{ij} )=||\bm{x}_{ij}||-r_s
\end{gathered}
\end{equation}
Where, $r_s>0$ is the minimum safe distance that the system must observe according to the CBF. $r_s$ is a parameter that can be determined by the designer.
Because $h_0 (\bm{q}_{ij})$ has a relative degree of two, we cannot directly apply general CBF methods. According to ECBF \cite{c21}, we define a new safety set, $C$.
\begin{equation}
\begin{gathered}
C= \{\bm{q}_{ij} \in \mathbb{R}^n | h (\bm{q}_{ij} )\geq 0\}\\
h(\bm{q}_{ij})=\dot{h}_0 (\bm{q}_{ij}) T_c+h_0 (\bm{q}_{ij})=||\bm{x}_{ij}||+\frac{\dot{\bm{x}}_{ij}^{\top}\bm{x}_{ij}}{||\bm{x}_{ij}||}T_c-r_s
\end{gathered}
\end{equation}
Where  $T_c>0$ denotes a constant in the temporal domain. From equation(1) which represents the systems of robot$_i$ and robot$_j$, we obtain system of $\bm{q}_{ij}$ as (9).
\begin{equation}
\begin{gathered}
\dot{\bm{q}}_{ij}=f(\bm{q}_{ij})+g(\bm{q}_{ij})(\bm{u}_j-\bm{u}_i)\\
Let \ f(\bm{q}_{ij})=[\dot{\bm{x}}_{ij}^{\top},0,0]^\top,g(\bm{q}_{ij} )=\begin{bmatrix}0&0&\frac{1}{m}&0\\ 0&0&0&\frac{1}{m} \end{bmatrix}^\top.
\end{gathered}
\end{equation}
We can rewrite the inequality constraints in (4) as (10).
\begin{equation}
L_f h(\bm{q}_{ij}) + L_g h(\bm{q}_{ij})(\bm{u}_j-\bm{u}_i) + \gamma h(\bm{q}_{ij}) \geq 0
\end{equation}
Where $\gamma>0$ is a constant parameter that determines the collision avoidance behavior. If $\gamma$ is small, the robot gradually reduces its speed from a distance sufficiently farther than the safe distance. If it is large, the robot brakes hard.
Because a robot cannot know the inputs of other robots, it cannot directly solve the inequality constraint. Therefore, we must decentralize the inequalities for each robot. Let us consider the sign dependence of $h(\bm{q}_{ij}),L_g h(\bm{q}_{ij})$ and $L_f h(\bm{q}_{ij})$ on $\bm{q}_{ij}$, for decentralization.
\begin{equation}
\begin{split}
h(\bm{q}_{ij})&=||\bm{x}_{ij}||+\frac{\dot{\bm{x}}_{ij}^{\top}\bm{x}_{ij}}{||\bm{x}_{ij}||}T_c-r_s\\
&=||-\bm{x}_{ij}||+\frac{(-\dot{\bm{x}}_{ij}^{\top})(-\bm{x}_{ij})}{||-\bm{x}_{ij}||}T_c-r_s=h(-\bm{q}_{ij})
\end{split}
\end{equation}
\begin{equation}
\begin{split}
L_g h(\bm{q}_{ij})&=\frac{\partial h(\bm{q}_{ij})}{\partial\bm{q}_{ij}}g(\bm{q}_{ij})=\left[\frac{\partial h(\bm{q}_{ij})}{\partial\bm{x}_{ij}},\frac{\partial h(\bm{q}_{ij})}{\partial \dot{\bm{x}}_{ij}}\right]\begin{bmatrix}0,0,\frac{1}{m},0\\ 0,0,0,\frac{1}{m} \end{bmatrix}^\top\\
&=\frac{\bm{x}_{ij}^\top}{m||\bm{x}_{ij}||}T_c
=-\frac{(-\bm{x}_{ij})^\top}{m||-\bm{x}_{ij}||}T_c
=-L_g h(-\bm{q}_{ij})
\end{split}
\end{equation}
\begin{equation}
\begin{gathered}
L_f h(\bm{q}_{ij})=\frac{\partial h(\bm{q}_{ij})}{\partial\bm{q}_{ij}} f(\bm{q}_{ij})=\left[\frac{\partial h(\bm{q}_{ij})}{\partial\bm{x}_{ij}},\frac{\partial h(\bm{q}_{ij})}{\partial \dot{\bm{x}}_{ij}}\right][\dot{\bm{x}}_{ij}^{\top},0,0]^\top\\
=\frac{\dot{\bm{x}}_{ij}^\top \bm{x}_{ij}}{||\bm{x}_{ij}||}+\frac{(y_{ij}\dot{x}_{ij}-x_{ij}\dot{y}_{ij})^2}{||\bm{x}_{ij}||^3}T_c\\
=\frac{(-\dot{\bm{x}}_{ij}^\top) (-\bm{x}_{ij})}{||-\bm{x}_{ij}||}+\frac{((-y_{ij})(-\dot{x}_{ij})-(-x_{ij})(-\dot{y}_{ij}))^2}{||-\bm{x}_{ij}||^3}T_c
=L_f h(-\bm{q}_{ij})
\end{gathered}
\end{equation}
The dependence of the sign on $\bm{q}_{ij}$ is represented by (14), (15), and (16).
\begin{equation}
h(\bm{q}_{ij})=h(-\bm{q}_{ij})
\end{equation}
\begin{equation}
     L_g h(\bm{q}_{ij})=-L_g h(-\bm{q}_{ij})
\end{equation}
\begin{equation}
L_f h(\bm{q}_{ij})=L_f h(-\bm{q}_{ij})
\end{equation}
Using $\bm{q}_{ji}=-\bm{q}_{ij}$, we decentralize the inequality constraint in (10) as (17).
\begin{equation}
\begin{gathered}
L_f h(\bm{q}_{ij}) - 2L_g h(\bm{q}_{ij})\bm{u}_i + \gamma h(\bm{q}_{ij}) \geq 0\\
L_f h(\bm{q}_{ji}) - 2L_g h(\bm{q}_{ji})\bm{u}_j + \gamma h(\bm{q}_{ji}) \geq 0\\
\end{gathered}
\end{equation}
Finally, we determine the control input $\bm{u}_i$ of robot$_i$ in the QP form, as (18).
\begin{equation}
\begin{gathered}
\underset {\bm{u}_i\in \mathbb{R}^m}{\mathrm{argmin}}{\frac{1}{2}||\bm{\hat{u}}_i-\bm{u}_i||^2}\\
L_f h(\bm{q}_{ij}) - 2L_g h(\bm{q}_{ij})\bm{u}_i + \gamma h(\bm{q}_{ij}) \geq 0\\
\end{gathered}
\end{equation}
\subsection{Objective of this Study}
We have observed several situations wherein the movement of three or more autonomous robots violates the constraints of decentralized ZCBF.
Because these constraints represent the requirements for ensuring collision avoidance, we require a systematic design and method for obtaining all possible inputs, in order to satisfy them.
However, in situations  where robots congregate at one location, the conventional distributed ZCBF may lose its possible input. We theoretically investigate this problem, and propose a modification for the distributed ZCBF to guarantee collision avoidance. Through analyses and simulations, we demonstrate the effectiveness of the proposed method.
First, to simplify the problem, we begin our analysis from a one-dimensional (1D) situation, where the robots move only in the straight direction. Then, we confirm that the proposed method can find a solution while avoiding collisions by simulating a 2D situation wherein  robots congregate at one location. Finally, we confirm the effectiveness of a real environment through an experiment which multiple robots congregate at one location.
\section{Theoretical Analysis of the No-Solution Scenario}
\subsection{Analysis of the Existence of the Solution}
We now consider how a no-solution situation occurs. Let us consider the situation in Fig. 1, where three robots are at the left end, center, and right end. The three robots move only along the horizontal axis. Their initial positions and inputs are presented in (19) and (20), where $[x_i(t),\dot{x}_i(t) ]^\top$ represents the position and velocity of robot$_i$ at time $t$.
   \begin{figure}[ht]
      \centering
      \includegraphics[width=0.8\textwidth]{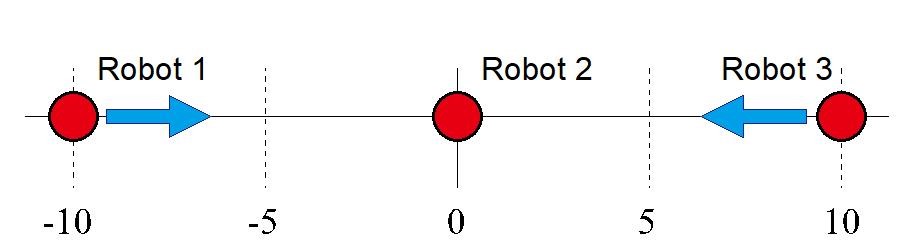}
      \caption{Example of a case of collision avoidance.}
   \end{figure}
\begin{equation}
\begin{split}
[x_1(0),\dot{x}_1(0) ]^\top&=[-10,5]^\top,\\
[x_2(0),\dot{x}_2(0) ]^\top&=[0,0]^\top,\\
[x_3(0),\dot{x}_3(0) ]^\top&=[10,-5]^\top
\end{split}
\end{equation}
\begin{equation}
\hat{u}_1=\hat{u}_2=\hat{u}_3=0
\end{equation}
This situation is similar to the case wherein $c=k=0$ in (1).
We assumed that the robot is a point mass, and that the safe distance $r_s=0.5$. At the initial positions, the distances between the robots are much greater than $r_s$. We provide constant velocities to robots 1 and 3, so that they constantly approach the position of robot 2.
First, we consider the constraint inequalities of robot 2, which are subject to the constraint inequalities between robots 1 and 3. Based on (17), the inequalities are (21).
\begin{equation}
\begin{gathered}
L_f h(\bm{q}_{21}) - 2L_g h(\bm{q}_{21})\bm{u}_2 + \gamma h(\bm{q}_{21}) \geq 0\\
L_f h(\bm{q}_{23}) - 2L_g h(\bm{q}_{23})\bm{u}_2 + \gamma h(\bm{q}_{23}) \geq 0\\
\end{gathered}
\end{equation}
Because robots 1 and 3 are symmetrically located, (22)-(25) hold true, based on the initial state and the safe distance design of the CBF.
\begin{equation}
x_{23}>0
\end{equation}
\begin{equation}
\bm{q}_{21}=-\bm{q}_{23}
\end{equation}
\begin{equation}
|x_{23}|\geq r_s
\end{equation}
\begin{equation}
L_g h(\bm{q}_{23})=\frac{\partial h}{\partial\bm{q}_{23}}g(\bm{q}_{23})=\frac{x_{23}}{m||x_{23}||}T_c=\frac{T_c}{m}>0
\end{equation}
According to (14), (15), and (16), we rewrite the constraint inequalities in (21) as (26).
\begin{equation}
\begin{gathered}
L_f h(\bm{q}_{23}) + 2L_g h(\bm{q}_{23})\bm{u}_2 + \gamma h(\bm{q}_{23}) \geq 0\\
L_f h(\bm{q}_{23}) - 2L_g h(\bm{q}_{23})\bm{u}_2 + \gamma h(\bm{q}_{23}) \geq 0\\
\end{gathered}
\end{equation}
Therefore, the range of input $u_2$ is (27).
\begin{equation}
\frac{L_fh(\bm{q}_{23})+\gamma h(\bm{q}_{23})}{2L_gh(\bm{q}_{23})}\geq u_2\geq-\frac{L_f h(\bm{q}_{23})+\gamma h(\bm{q}_{23})}{2L_gh(\bm{q}_{23})}
\end{equation}
As $L_g h(\bm{q}_{23})>0$, the condition for the existence of $u_2$  depends on inequality $(28)$, given below:
\begin{equation}
L_fh(\bm{q}_{23})+\gamma h(\bm{q}_{23})\geq0
\end{equation}
If $L_fh(\bm{q}_{23})+\gamma h(\bm{q}_{23})$ is less than zero, the existence of a solution is not guaranteed, which unfortunately, can occur. We prove this in Theorem 1.
\begin{theorem}[]\label{thm1}
If the robots follow the initial conditions in (19) and constraint inequalities in (21), the minimum value of $L_fh(\bm{q}_{23})+\gamma h(\bm{q}_{23})<0$
\end{theorem}

\begin{proof}[Proof of Theorem~{\upshape\ref{thm1}}]
 We can write $L_fh(\bm{q}_{23})+\gamma h(\bm{q}_{23})$ as in (29).
\begin{equation}
L_fh(\bm{q}_{23})+\gamma h(\bm{q}_{23})
=\frac{\dot{x}_{23}^{}x_{23}}{||x_{23}||}+\gamma(||x_{23}||+\frac{\dot{x}_{23}^{}x_{23}}{||x_{23}||}T_c-r_s )
\end{equation}
Using (3) and (8), we obtain (30). Here, we obtain the relationship between $\frac{\dot{x}_{ij}^{}x_{ij}}{||x_{ij}||}$ and $||x_{ij}||$ using the time derivative instead of the Lie derivative, which uses the inputs of the system. Additionally, substituting (1) into the acceleration term of the time derivative of $h(\bm{q}_{ij} )$ yields the same result as the Lie derivative.
\begin{equation}
\dot{h}(\bm{q}_{ij} )=\frac{\dot{x}_{ij}^{}x_{ij}}{||x_{ij}||}+\frac{d}{dt} \left(\frac{\dot{x}_{ij}^{}x_{ij}}{||x_{ij}||}\right)T_c
\end{equation}
From (3) and (30), we obtain $\dot{h} (\bm{q}_{ij} )+\gamma h(\bm{q}_{ij})$ in (31).
\begin{equation}
\begin{gathered}
\dot{h} (\bm{q}_{ij} )+\gamma h(\bm{q}_{ij}) \\
=\frac{\dot{x}_{ij}^{}x_{ij}}{||x_{ij}||}+\frac{d}{dt} \left(\frac{\dot{x}_{ij}^{}x_{ij}}{||x_{ij}||}\right)T_c+\gamma(||x_{ij}||+\frac{\dot{x}_{ij}^{}x_{ij}}{||x_{ij}||}T_c-r_s )\geq0
\end{gathered}
\end{equation}
Satisfying $\dot{h} (\bm{q}_{ij} )+\gamma h(\bm{q}_{ij})\geq0$ ensures that $h(\bm{q}_{ij})\geq0$.
$h(\bm{q}_{ij})\geq0$ and (8) gives the relation between $\frac{\dot{x}_{ij}^{}x_{ij}}{||x_{ij}||}$ and $||x_{ij}||$ as in (32).
\begin{equation}
\frac{\dot{x}_{ij}^{}x_{ij}}{||x_{ij}||}\geq -\frac{||x_{ij} ||-r_s}{T_c}
\end{equation}
In addition, by setting $h_2(\bm{q}_{ij})=\gamma T_c(||x_{ij}||-r_s)+\frac{\dot{x}_{ij}^{}x_{ij}}{||x_{ij}||}T_c$,
 (31) can be transformed into (33).
\begin{equation}
\begin{gathered}
\dot{h} (\bm{q}_{ij} )+\gamma h(\bm{q}_{ij}) \\
=\frac{\dot{x}_{ij}^{}x_{ij}}{||x_{ij}||}+\frac{d}{dt} \left(\frac{\dot{x}_{ij}^{}x_{ij}}{||x_{ij}||}\right)T_c
+\gamma(||x_{ij}||+\frac{\dot{x}_{ij}^{}x_{ij}}{||x_{ij}||}T_c-r_s )
\\=\frac{\dot{x}_{ij}^{}x_{ij}}{||x_{ij}||}\gamma T_c+\frac{d}{dt} \left(\frac{\dot{x}_{ij}^{}x_{ij}}{||x_{ij}||}\right)T_c
+\frac{1}{T_c} \left(\gamma T_c(||x_{ij}||-r_s)+\frac{\dot{x}_{ij}^{}x_{ij}}{||x_{ij}||}T_c\right)\\
=\dot{h}_2(\bm{q}_{ij})+\frac{1}{T_c}h_2(\bm{q}_{ij})\geq0
\end{gathered}
\end{equation}
Because (33) is the same as (3), $h_2(\bm{q}_{ij})\geq0$ is also guaranteed by satisfying $\dot{h} (\bm{q}_{ij} )+\gamma h(\bm{q}_{ij})\geq0$.
Therefore, because both $h(\bm{q}_{ij})\geq0$ and $h_2(\bm{q}_{ij})\geq0$ are always satisfied, the relation between $\frac{\dot{x}_{ij}^{}x_{ij}}{||x_{ij}||}$ and $||x_{ij}||$
is as in (34).
\begin{equation}
\frac{\dot{x}_{ij}x_{ij}}{||x_{ij}||}\geq \mathrm{max}\left(-\frac{||x_{ij}||-r_s}{T_c},-\gamma(||x_{ij}||-r_s)\right)
\end{equation}
Using (34), we find the minimum value of $L_fh(\bm{q}_{23})+\gamma h(\bm{q}_{23})$.
\begin{equation}
\begin{split}
\mathrm{min}\ L_fh(\bm{q}_{23})+\gamma h(\bm{q}_{23})
=\mathrm{max}\left(\frac{\dot{x}_{23}^{}x_{23}}{||x_{23}||},\frac{\dot{x}_{23}^{}x_{23}}{||x_{23}||}\gamma T_c\right)
\end{split}
\end{equation}
From the initial conditions and parameters, because we can confirm that $\dot{x}_{23}<0$,$x_{23}>0$ and $\gamma,T_c>0$, we prove that
$\frac{\dot{x}_{23}^{}x_{23}}{||x_{23}||}\gamma T_c,\frac{\dot{x}_{23}^{}x_{23}}{||x_{23}||}<0$.
Thus, we demonstrate that the minimum value of $L_fh(\bm{q}_{23})+\gamma h(\bm{q}_{23})$ is less than 0.
\end{proof}
This indicates that the existence of input $u_2$ is not guaranteed in this problem.
\subsection{Verification by 1D Simulation}
We verify the analytical results using simulations. We used “MatlabR2021a" from MathWorks as the simulation environment. The control period was set to $0.025$s and the simulation was performed for up to $320$ steps ($8$s). We used the “fmincon” function in MATLAB to compute the quadratic programs required for CBF processing. The parameters used in the simulations, which are based on \cite{c14},\cite{c19}, are listed in Table 1.
Fig.2 illustrates the possible range of the input for robot 2 $u_2$, in which the red and blue curves indicate the upper and lower limits of $u_2$, respectively. To determine the possible input, the red curve must always be superior to the blue curve. However, between $t=1.5$ and $3$s, the blue curve is higher than the red curve. This shows the inability to find an input that meets the conditions required to ensure collision avoidance.
We hypothesize that this problem can be resolved by modifying the decentralization procedure for CBF. In Section 4, we propose asymmetric decentralization to compensate for the lack of knowledge of the others robot's inputs.
\begin{table}[ht]
\caption{Parameters Specification}\label{<table-label1>}%
\begin{tabular}{@{}lll@{}}
\toprule
Parameter & Value & Name\\
\midrule
$k$ &$1$ & Position feedback coefficient\\
$c$ & $0.3$ & Damper coefficient\\
$r_s$ & 0.5 & Safe distance\\
$\gamma$ & 2 & Constraint parameters\\
$m$ & 1 & Mass\\
$T_c$ & $0.025$ & Time constant\\
\botrule
\end{tabular}
\end{table}
   \begin{figure}[ht]
      \centering
      \includegraphics[width=0.65\textwidth]{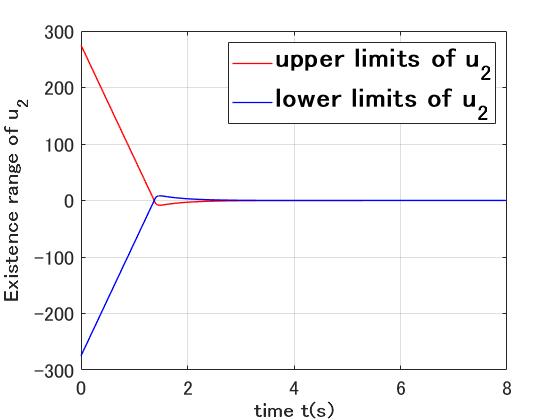}
      \caption{Existence range of the solution of $u_2$.}
   \end{figure}
\section{Asymmetric CBF to Ensure the Safety Condition}
\subsection{Asymmetric Design by Weight Functions}
We propose a method for avoiding the "no-solution" situation described in Section III by introducing asymmetric dynamic weight functions.
In this chapter, we analyze  whether the robot can move in a 2D plane.
Based on our proposed method, we introduce a weighted ZCBF\cite{c17} (36). The existence of solutions had not been discussed in previous research\cite{c17}. In the original weighted ZCBF, the constraint can be decentralized, as in inequality (37). If $w_{1 ij}+w_{1 ji}=2$ and $w_{2 ij}+w_{2 ji}=2$ holds, (36) becomes equivalent to (10). If we select $w_{1 ij}=w_{1 ji}=w_{2 ij}=w_{2 ji}=1$, (37) becomes (17).
\begin{equation}
 \frac{(w_{1 ij}+w_{1 ji})}{2}L_f h(\bm{q}_{ij}) + L_g h(\bm{q}_{ij})(\bm{u}_j-\bm{u}_i) + \frac{(w_{2 ij}+w_{2 ji})}{2}\gamma h(\bm{q}_{ij})  \geq 0
\end{equation}
\begin{equation}
\begin{gathered}
w_{1 ij}L_f h(\bm{q}_{ij}) - 2L_g h(\bm{q}_{ij})\bm{u}_i + w_{2 ij}\gamma h(\bm{q}_{ij}) \geq 0\\
w_{1 ji}L_f h(\bm{q}_{ji}) - 2L_g h(\bm{q}_{ji})\bm{u}_j + w_{2 ji}\gamma h(\bm{q}_{ji}) \geq 0\\
\end{gathered}
\end{equation}
In this study, we introduce a novel design of weight functions for $w_{1 ij}$, $w_{1 ji}$, $w_{2 ij}$ and $w_{2 ji}$ in (36) and analytically show that the weight functions can guarantee satisfying (10).
Equation (38) and (39) show the proposed weight function using the absolute velocity of the robot. As shown in (40) and (41), the weight functions satisfy $w_{1 ij}+w_{1 ji}=2$ and $w_{2 ij}+w_{2 ji}=2$. When $\dot{x}_i \neq \dot{x}_j$, $w_{1 ij}$, $w_{1 ji}$ and $w_{2 ij}$, $w_{2 ji}$ will be different. Subsequently, the CBF constraint will be asymmetrically decentralized. By substituting (38) and (39) into (37), we obtain (42) for the robot$_i$ to guarantee collision avoidance:
\begin{equation}
w_{1 ij}=\frac{-2\frac{\dot{\bm{x}}_{i}^{\top}\bm{x}_{ij}}{||\bm{x}_{ij}||}+\frac{(y_{ij}\dot{x}_{ij}-x_{ij}\dot{y}_{ij})^2}{||\bm{x}_{ij}||^3}T_c}{L_f h(\bm{q}_{ij})}\\
\end{equation}
\begin{equation}
w_{2 ij}=\frac{||\bm{x}_{ij}||-2\frac{\dot{\bm{x}}_{i}^{\top}\bm{x}_{ij}}{||\bm{x}_{ij}||}T_c-r_s}{h(\bm{q}_{ij})}\\
\end{equation}
\begin{equation}
\begin{gathered}
w_{1 ij}+w_{1 ji}=\frac{-2\frac{\dot{\bm{x}}_{i}^{\top}\bm{x}_{ij}}{||\bm{x}_{ij}||}+\frac{(y_{ij}\dot{x}_{ij}-x_{ij}\dot{y}_{ij})^2}{||\bm{x}_{ij}||^3}T_c}{L_f h(\bm{q}_{ij})}+\frac{-2\frac{\dot{\bm{x}}_{j}^{\top}\bm{x}_{ji}}{||\bm{x}_{ji}||}+\frac{(y_{ji}\dot{x}_{ji}-x_{ji}\dot{y}_{ji})^2}{||\bm{x}_{ji}||^3}T_c}{L_f h(\bm{q}_{ji})}\\
=2\frac{\frac{\dot{\bm{x}}_{ij}^{\top}\bm{x}_{ij}}{||\bm{x}_{ij}||}+\frac{(y_{ij}\dot{x}_{ij}-x_{ij}\dot{y}_{ij})^2}{||\bm{x}_{ij}||^3}T_c}{L_f h(\bm{q}_{ij})}
=2\frac{L_f h(\bm{q}_{ij})}{L_f h(\bm{q}_{ij})}=2
\end{gathered}
\end{equation}
\begin{equation}
\begin{gathered}
w_{2 ij}+w_{2 ji}=\frac{||\bm{x}_{ij}||-2\frac{\dot{\bm{x}}_{i}^{\top}\bm{x}_{ij}}{||\bm{x}_{ij}||}T_c-r_s}{h(\bm{q}_{ij})}+\frac{||\bm{x}_{ji}||-2\frac{\dot{\bm{x}}_{j}^{\top}\bm{x}_{ji}}{||\bm{x}_{ji}||}T_c-r_s}{h(\bm{q}_{ji})}\\
=2\frac{||\bm{x}_{ij}||+\frac{\dot{\bm{x}}_{ij}^{\top}\bm{x}_{ij}}{||\bm{x}_{ij}||}T_c-r_s}{h(\bm{q}_{ij})}=2\frac{h(\bm{q}_{ij})}{h(\bm{q}_{ij})}=2
\end{gathered}
\end{equation}
\begin{equation}
-2\frac{\dot{\bm{x}}_{i}^{\top}\bm{x}_{ij}}{||\bm{x}_{ij}||}+\frac{(y_{ij}\dot{x}_{ij}-x_{ij}\dot{y}_{ij})^2}{||\bm{x}_{ij}||^3}T_c - 2L_g h(\bm{q}_{ij})\bm{u}_i + \gamma(||\bm{x}_{ij}||-2\frac{\dot{\bm{x}}_{i}^{\top}\bm{x}_{ij}}{||\bm{x}_{ij}||}T_c-r_s) \geq 0\\
\end{equation}\\
\newpage
By introducing these weight functions, we can guarantee a common solution to all the simultaneous inequalities. We prove this in Theorem 2.\\
\begin{theorem}[]\label{thm2}
 A control input $\bm{u}_i=-m(\gamma+\frac{1}{T_c}) \dot{\bm{x}}_i$, always satisfies inequality (42).
\end{theorem}
\begin{proof}[Proof of Theorem~{\upshape\ref{thm2}}]
By substituting $\bm{u}_i=-m(\gamma+\frac{1}{T_c}) \dot{\bm{x}}_i$ into (42), (43) holds true. 
\begin{equation}
\begin{gathered}
-2\frac{\dot{\bm{x}}_{i}^{\top}\bm{x}_{ij}}{||\bm{x}_{ij}||}+\frac{(y_{ij}\dot{x}_{ij}-x_{ij}\dot{y}_{ij})^2}{||\bm{x}_{ij}||^3}T_c - 2L_g h(\bm{q}_{ij})\bm{u}_i + \gamma(||\bm{x}_{ij}||-2\frac{\dot{\bm{x}}_{i}^{\top}\bm{x}_{ij}}{||\bm{x}_{ij}||}T_c-r_s)\\
=-2\frac{\dot{\bm{x}}_{i}^{\top}\bm{x}_{ij}}{||\bm{x}_{ij}||}+\frac{(y_{ij}\dot{x}_{ij}-x_{ij}\dot{y}_{ij})^2}{||\bm{x}_{ij}||^3}T_c - 2\frac{\bm{x}_{ij}^\top}{m||\bm{x}_{ij}||}T_c(-m(\gamma+\frac{1}{T_c})\dot{\bm{x}}_i)\\+\gamma(||\bm{x}_{ij}||-2\frac{\dot{\bm{x}}_{i}^{\top}\bm{x}_{ij}}{||\bm{x}_{ij}||}T_c-r_s)\\
=\frac{(y_{ij}\dot{x}_{ij}-x_{ij}\dot{y}_{ij})^2}{||\bm{x}_{ij}||^3}T_c+\gamma(||\bm{x}_{ij}||-r_s)
\end{gathered}
\end{equation}
From Equation (7), since $||\bm{x}_{ij}||-r_s\geq0$, we have $\frac{(y_{ij}\dot{x}_{ij}-x_{ij}\dot{y}_{ij})^2}{||\bm{x}_{ij}||^3}T_c+\gamma(||\bm{x}_{ij}||-r_s)\geq0$.
This implies that the input always satisfies the condition guaranteeing collision avoidance. 
This input is the same as in (1), where $c = m(\gamma+\frac{1}{T_c})$ and $k = 0$.
Because the input $\bm{u}_i=-m(\gamma+\frac{1}{T_c}) \dot{\bm{x}}_i$ does not include any values of the relative distance $\bm{x}_{ij}$, relative velocity $\bm{\dot{x}}_{ij}$, or safe distance $r_s$, it is a common solution for all coupled constraints of the robot${_i}$.
\end{proof}
Thus, the existence of a solution can be guaranteed by introducing asymmetric weight functions.
Because the proposed method does not change the original constraints of CBF but only modifies the decentralization method, it is valid for safety control, as demonstrated in previous CBF studies [12-19].
\subsection{Verification by 1D Simulation}
The proposed method was applied in the simulation of the scenario shown in Fig. 1. All parameters were the same as those described in Section 3. Fig. 3 shows the upper (red) and lower (blue) limits of a possible solution for robot 2. In contrast to the results shown in Fig. 2, there were no conflicts between the upper and lower limits. The safety conditions were maintained during the simulation.
\clearpage
\begin{figure}[ht]
   \centering
   \includegraphics[width=0.65\textwidth]{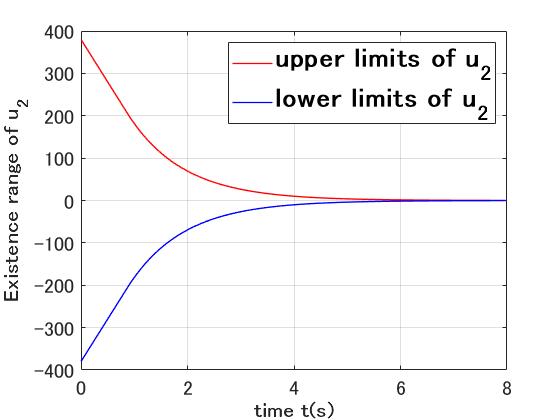}
   \caption{Existence range of the solution of $u_2$.}
\end{figure}

\section{Applicability to 2D Search Problems}
We applied the proposed method to a 2D field. We performed simulations for verification by considering collision avoidance.
We set the simulation conditions as shown in Fig. 4. In the initial state, we placed 20 robots in a circular formation as indicated by the red dots in Fig. 4.
All robots gathered toward the green dot in the center of Fig. 4.
In such a situation, where many robots congregate at one location, collision avoidance becomes necessary.
We performed simulations under three conditions: symmetric CBF, Asymmetric CBF with the weight function proposed in the previous research\cite{c17} as shown in equation (44) and (45), and Asymmetric CBF with the weight function of the proposed method.
\begin{equation}
w_{1 ij}=1
\end{equation}\\
\begin{equation}
w_{2 ij}=\frac{||\dot{\bm{x}}_{j}||}{||\dot{\bm{x}}_{i}||+||\dot{\bm{x}}_{j}||}
\end{equation}\\
The simulation parameters are listed in Table 2. We performed simulations by changing the CBF parameter $\gamma$, which could change the behavior of the robots, as discussed in Section 2. We performed the simulations 100 times for each value of $\gamma$. The initial positions and velocities of the robots exhibited microrandomness in each simulation.
\begin{figure}[ht]
   \centering
   \includegraphics[width=0.8\textwidth]{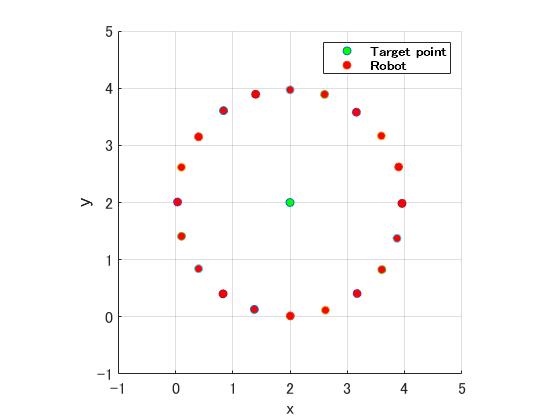}
   \caption{initial state of simulation}
\end{figure}
\begin{table}[ht]
\caption{Parameters Specification}\label{<table-label3>}%
\begin{tabular}{@{}lll@{}}
\toprule
Parameter & Value & Name\\
\midrule
Parameter & Value & Name\\
$k$ &$1$ & Position feedback coefficient\\
$c$ & $0.3$ & Damper coefficient\\
$r_s$ & $0.08$ & Safe distance\\
$m$ & 1 & Mass\\
$T_c$ & $0.025$ & Time constant\\
\botrule
\end{tabular}
\end{table}
\newpage
We examined whether a no-solution state occurred or not, with and without the application of the proposed method. We declared that there was no solution if the “fmincon” function of MATLAB, used for solving the QP, failed to find a solution. We set a constant tolerance of $10^{-6}$.
The simulation was performed for up to 400 steps (10s). We plotted the minimum constraint inequality (42) for each step in Fig. 5, 6 and 7 using symmetric, asymmetric weight functions in previous research and asymmetric weight proposed in this study. In the conventional method (symmetric and asymmetric weight functions in previous research), the minimum value was less than 0.
However, in the proposed method (asymmetric), the minimum value was greater than zero in all steps. Table 3 lists the number of no-solution situations in the 100 simulations for each condition. With the proposed method, the robots successfully maintained their safety conditions in all simulations; however, they could not achieve this, in any case, without applying the proposed method. The results demonstrate that the proposed method ensures collision avoidance in 2D searching tasks.
\begin{figure}[ht]
 \begin{minipage}[b]{0.5\linewidth}
 \centering
 \includegraphics[width=0.9\textwidth]{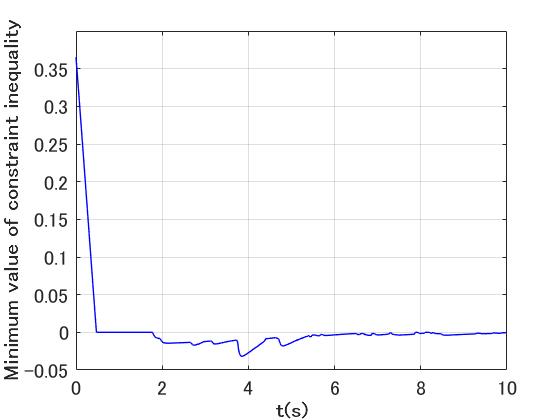}
 \caption{Minimum value of constraint inequality (symmetry)}
 \end{minipage}
 \begin{minipage}[b]{0.5\linewidth}
 \centering
 \includegraphics[width=0.9\textwidth]{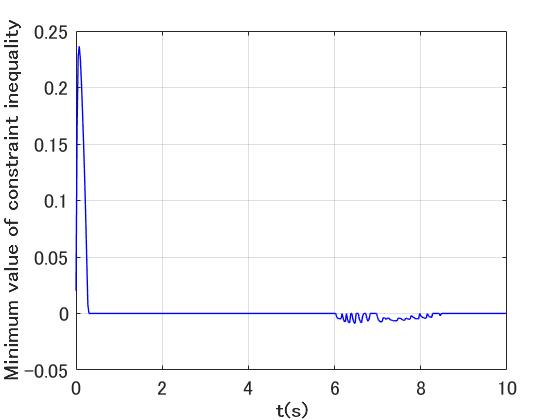}
 \caption{Minimum value of constraint inequality (asymmetry,previous research[14])}
 \end{minipage}
 \end{figure}
 \begin{figure}[ht]
 \centering
 \includegraphics[width=0.65\textwidth]{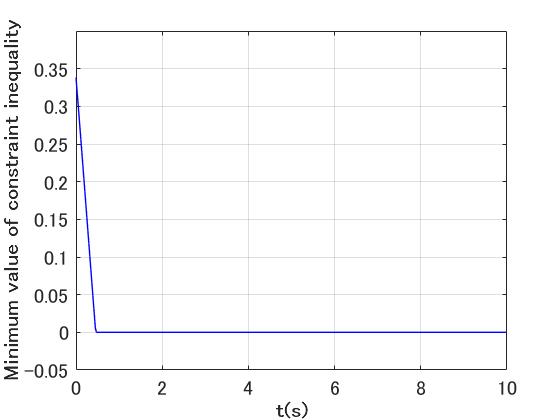}
 \caption{Minimum value of constraint inequality (asymmetry,proposed)}
 \end{figure}
 \begin{table}[ht]
 \caption{NUMBER OF NO-SOLUTION SITUATIONS IN 100 SIMULATIONS}\label{<table-label2>}%
 \begin{tabular}{@{}llll@{}}
 \toprule
$\gamma$ & Asymmetric CBF (Proposed) & Asymmetric CBF (Previous research) & Symmetric CBF\\
 \midrule
 $0.5$ & 0 & 100 & 74\\
 $1.0$ & 0 & 100 & 85\\
 $5.0$ & 0 & 100 & 100\\
 \botrule
 \end{tabular}
 \end{table}
 \newpage
\section{Experiments with real robots}
To verify the effectiveness of the proposed method in a real environment, we conduct an experiment with the real mobile robot shown in Fig. 8. 
The robot is fully distributed, equipped with a 2D LiDAR, which as 240 degree field of view angle. It has a reflector and four omni-wheels with motors.
The robot can obtain the relative position, relative velocity, and absolute velocity data necessary to constrain of CBF from sensor data and the encoders of the motors. 
 The robot updates and follows the target speed using the  acceleration input designed using the inequality constraint equation (42) at a control period of 10 Hz. 
 In the experiment, we set each parameter of equation (42) as follows: safety distance $r_s=400$mm, $T_c=0.1$, $\gamma=0.3$.
    \begin{figure}[ht]
 \centering \includegraphics[width=0.65\textwidth]{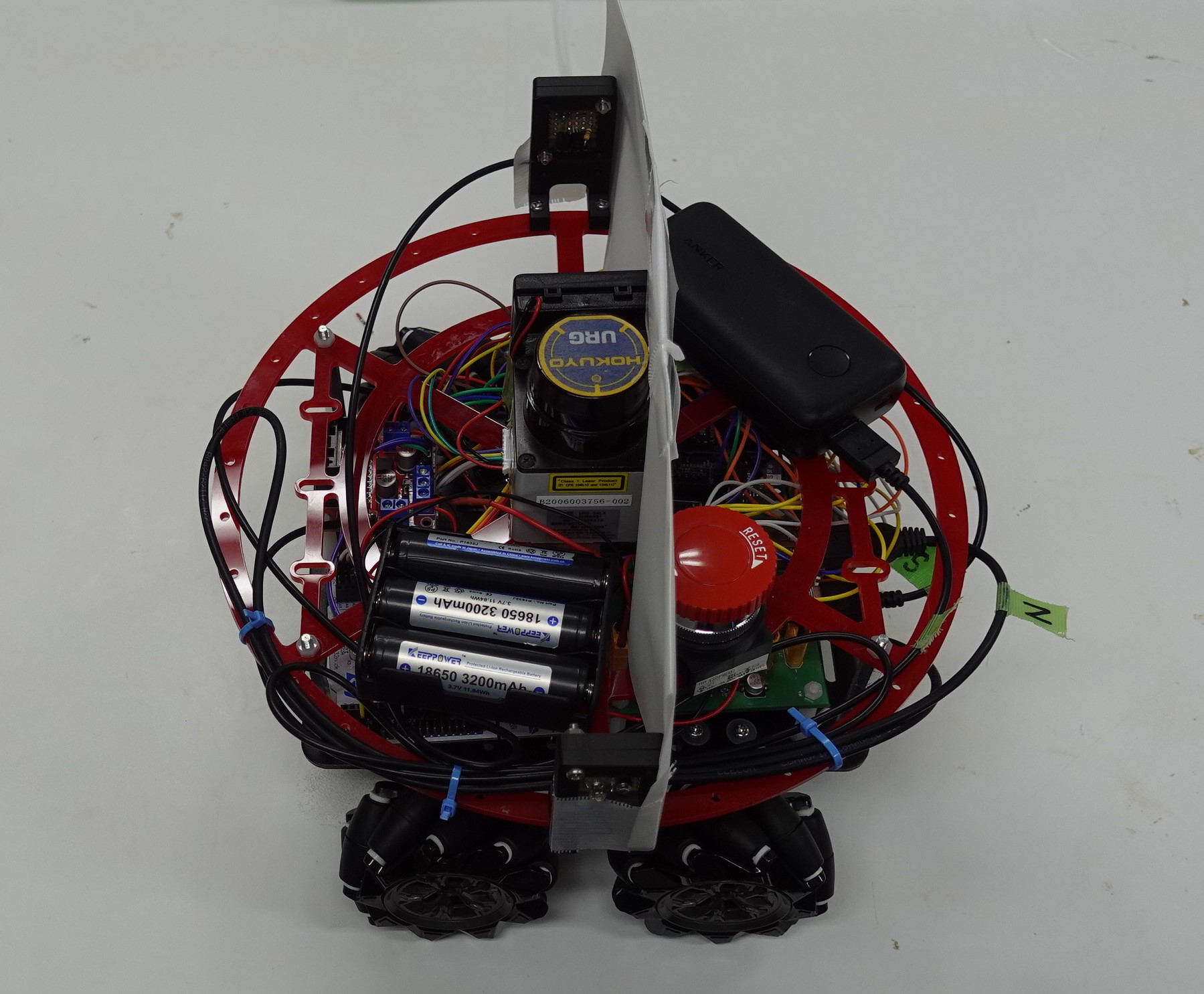}
 \caption{Robots used for experiments}
 \end{figure}\\
We set a similar condition to Fig 1. Fig. 9 shows the experimental environment with three robots. To simplify the experimental condition, we put guide rails, by which a robot just moves forward and backward. The left and right robot are going to approach the position of the center robot. For the left and right robots, we set the initial robot velocity $v_0=300$mm/s and the acceleration input $\hat{u}=50$mm/$s^2$ without considering collision avoidance for each step. For the center robot, we set $v_0=0$mm/s, $\hat{u}=0$mm/$s^2$.
  \begin{figure}[ht]
 \centering
 \includegraphics[width=0.65\textwidth]{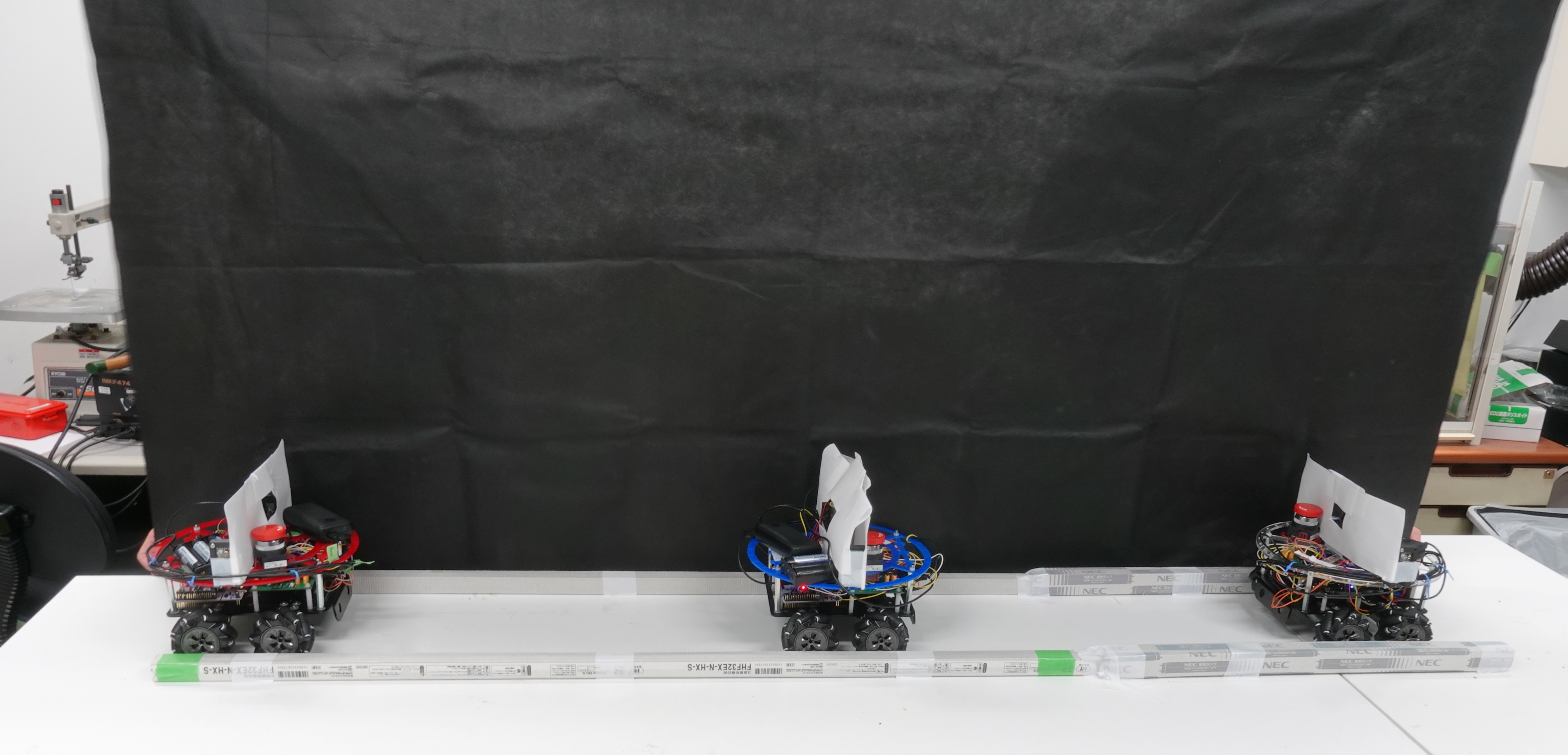}
 \caption{Experimental environment}
 \end{figure}
 \clearpage
 Fig. 10 and 11 show the upper and lower limit of possible input for the center robot considering CBF constraints. Fig 10 indicates the results by a symmetric weight function, where the lower limit (blue curve) became higher than the upper limit (orange). In contrast, the upper limit was higher or almost the same as the lower in Fig. 11. Due to sensing variance of the LiDAR, we found some crosses of the blue and orange curves. We consider the difference between Fig. 10 and 11 are significant and the proposed algorithm surely supressed no-solution situations.
 \begin{figure}[ht]
 \begin{minipage}[b]{0.5\linewidth}
 \centering
 \includegraphics[width=0.9\textwidth]{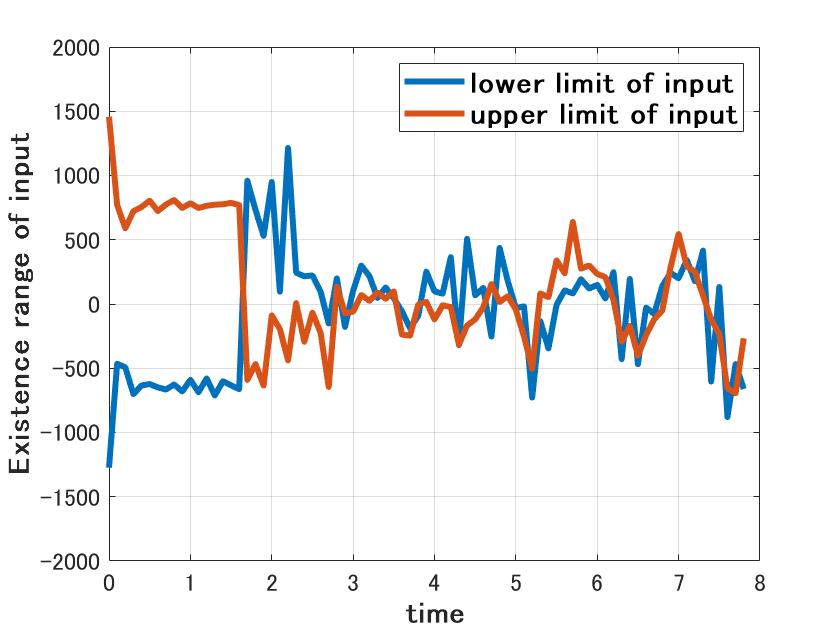}
 \caption{Existence range of input(symmetry)}
 \end{minipage}
 \begin{minipage}[b]{0.5\linewidth}
 \centering
 \includegraphics[width=0.9\textwidth]{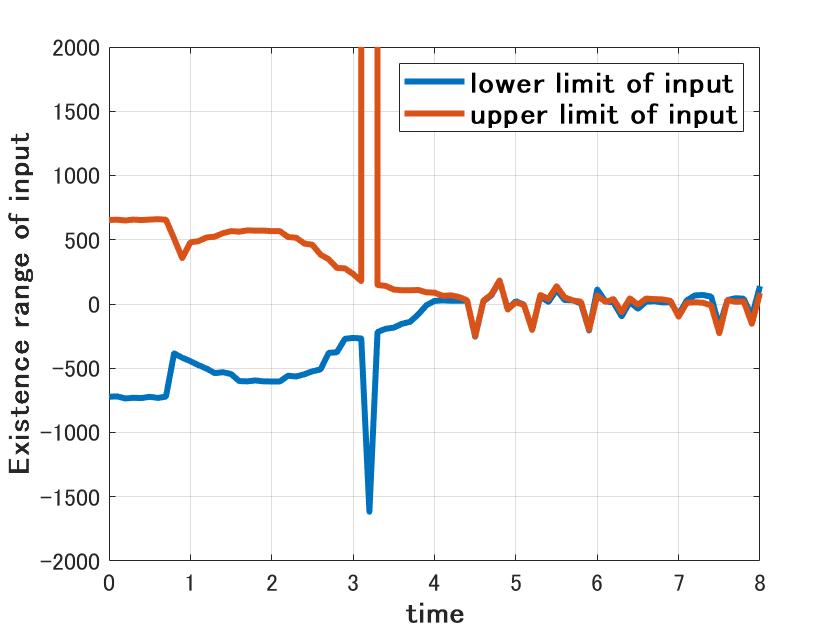}
 \caption{Existence range of input(asymmetry)}
 \end{minipage}
 \end{figure}
 
 We show the closest approach of the robots in the experiment with symmetric weight functions in Fig. 12. We confirmed the robots clearly violated the safety distance as a result of constraint violations. On the other hand, We show the closest approach of the robots in the experiment with the proposed asymmetric weight function in Fig. 13 and confirmed the robots was able to stop at a safe distance.
 From these results, we confirm the effectiveness of the proposed method in a real-world environment.
\begin{figure}[ht]
 \begin{minipage}[b]{0.5\linewidth}
 \centering
 \includegraphics[width=0.9\textwidth]{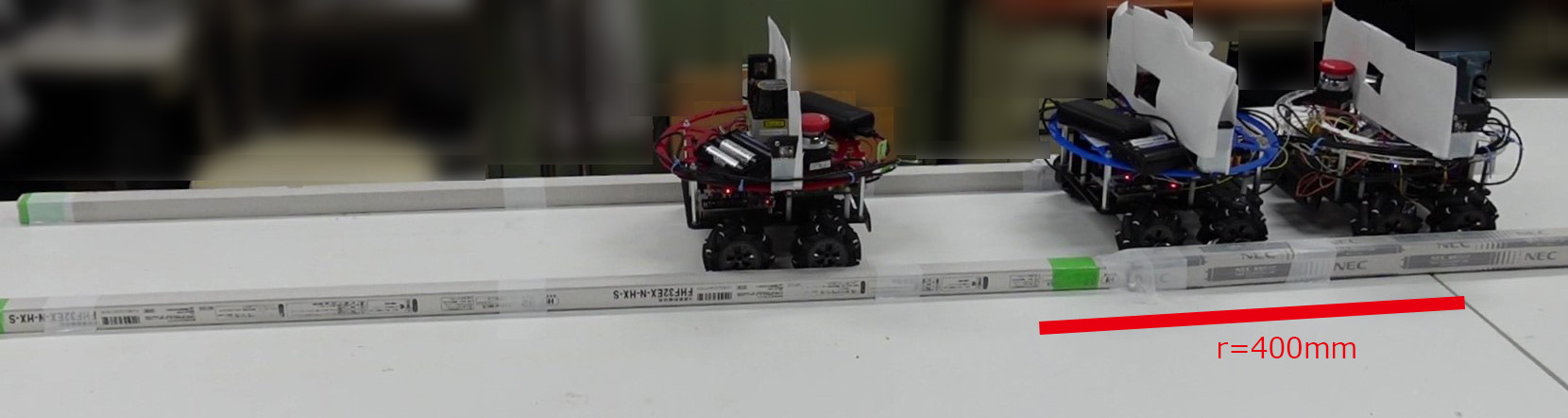}
 \caption{Robot at closest approach during experiment(symmetry)}
 \end{minipage}
 \begin{minipage}[b]{0.5\linewidth}
 \centering
 \includegraphics[width=0.9\textwidth]{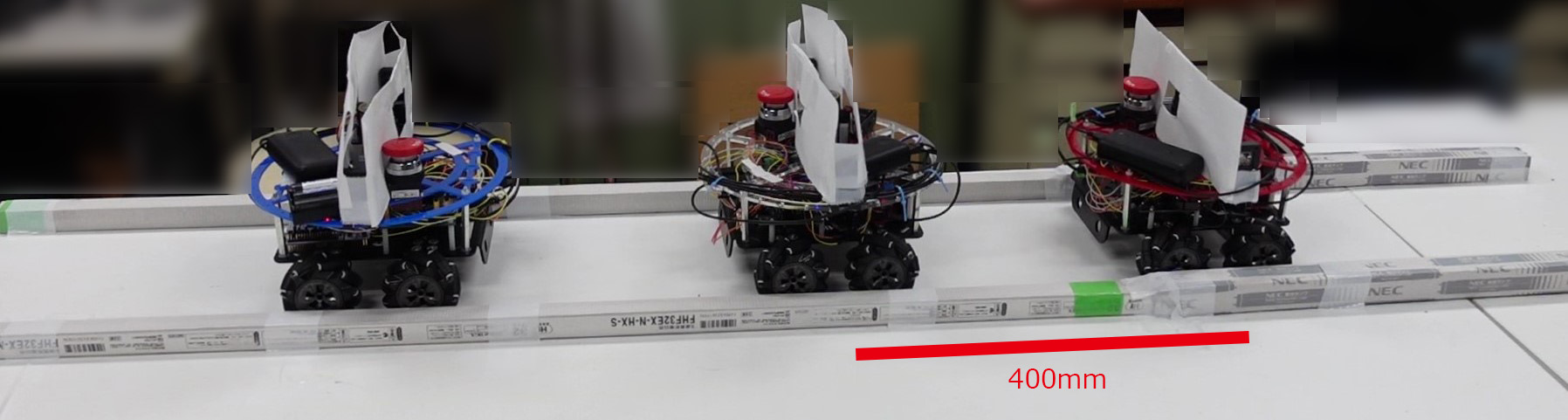}
 \caption{Robot at closest approach during experiment(asymmetry)}
 \end{minipage}
 \end{figure}
\section{Conclusions}
In this study, we proposed a collision avoidance technique using decentralized CBF for a swarm robotic system composed of robots that obeyed the equations of motion of a second-order differential system. We analyzed the problem of no solution in a conventional decentralized CBF, which revealed the need to revise the method to distribute constraints among robots.We then proposed an asymmetric method to decentralize the CBF conditions so that robots could avoid no-solution situations. We then validated the proposed method in 1D and 2D planes and real robots.
Our future plans include to work on robot implementation in two dimensions.
\section*{Acknowledgments}
This study was partially supported by Jana Society for the Promotion of Science "KAKENHI" grant-in-aid No. JP21K18967 and JP22H01440, and the Sensing Solution University Collaboration Program of the Sony Semiconductor Solutions Corporation.We thank Ms. Kanon Yokoi for assistance of experimental set up.
\begin{itemize}
\item Authors' contributions
H.E conceieved and designed the control method and simulations, and wrote the paper. Y.O and D.K analyzed and evaluated the proof and the results, and contributed to manuscript revision.
\item Funding
This study was partially supported by Jana Society for the Promotion of Science "KAKENHI" grant-in-aid No. JP21K18967 and JP22H01440, and the Sensing Solution University Collaboration Program of the Sony Semiconductor Solutions Corporation.
\end{itemize}
\section*{Declarations}
\begin{itemize}
\item Conflict of interest/Competing interests 
The authors declare no conflict of interest.
\item Ethics approval
The authors declare the manuscript doesn't include any concerns on ethical issues.
\item Consent to participate
This paper does not involve human participants and animals.
\item Consent for publication
H.E, Y.O and D.K declare they have agreed to submit and publsih the manuscript.
\item Availability of data and materials
The authors declare they will open the codes and simulation results on github after acceptance of the paper.
\item Code availability
The authors declare they will open the codes and simulation results on github after acceptance of the paper.
\end{itemize}


\begin{thebibliography}{99}

\bibitem{c1}
Rongxin, C., \textit{et al.}: Leader–follower formation control of under actuated autonomous underwater vehicles. Ocean Eng., vol. 37, pp. 1491–1502(2010) 10.1016/j.oceaneng.2010.07.006.
\bibitem{c2}
Tatsuya, M.,\textit{ et al.}: Distributed force/position optimization dynamics for cooperative unknown payload manipulation. in Proc. 2020 59th IEEE Conf. Decis. Control (CDC), pp. 5366–5373, Korea(2020) , 10.1109/CDC42340.2020.9304246.
\bibitem{c3}
Naohiko, S.,\textit{ et al.}: Collective motion in a system of motile elements. Phys. Rev. Lett., vol. 76, pp. 3870–3873(1996) 10.1103/PhysRevLett.76.3870.
\bibitem{c4}
Tomohisa, H.,\textit{ et al.}: Formation control of multi-agent systems with sampled information–Relationship between information exchange structure and control performance–. in Proc. 45th IEEE Conf.	Decis. Control (CDC), pp. 4333–4338, USA(2006) 10.1109/CDC.2006.377708.
\bibitem{c5}David, A., Pierre, A.B.: Stability of leaderless discrete-time multi-agent systems. Math. Control Signals Syst., vol. 18, pp. 293–322(2006) 10.1007/s00498-006-0006-0.
\bibitem{c6}
Gajamohan, M., Tomohisa, H.: Influence of stochastic communication loss on the stability of a formation of multiple agents. in Proc. 2007 Am. Control Conf., pp. 341–346, USA(2007) 10.1109/ACC.2007.4282877.
\bibitem{c7}
Dimos, V.D.,Kostas, J.K.: On the rendezvous problem for multiple nonholonomic agents. IEEE Trans. Autom. Control, vol. 52,pp. 916–922(2007) 10.1109/TAC.2007.895897.
\bibitem{c8}
Luciano, C.A.P.,\textit{ et al.}: Swarm coordination based on smoothed particle hydrodynamics technique.  IEEE Trans. Robot., vol. 29, pp. 383–399(2013) 10.1109/TRO.2012.2234294.
\bibitem{c9}
Daito, S.,\textit{ et al.}: Leader–follower navigation in obstacle environments while preserving connectivity without data transmission. IEEE Trans. Control Syst. Technol., vol. 26, pp. 1233–1248(2018) 10.1109/TCST.2017.2705121.
\bibitem{c10}
Kazuya, S.,\textit{ et al.}: Self-assembly through the local Interaction between “embodied” nonlinear oscillators with simple motile function. in Proc. 2008 IEEE/RSJ Int. Conf. Intell. Robots Syst., Nice, France, pp. 1319–1324(2008) 10.1109/IROS.2008.4650804.
\bibitem{c11}
Kwang, K.O.,\textit{ et al.}: A survey of multi-agent formation control. Automatica, vol. 53, pp. 424–440(2015) https://doi.org/10.1016/j.automatica.2014.10.022.
\bibitem{c12}
Aaron, D.A.,\textit{ et al.}: Control barrier function based quadratic programs with application to adaptive cruise control. In Proc. 53rd IEEE Conf. Decis. Control, pp. 6271–6278, USA(2014) 10.1109/CDC.2014.7040372.
\bibitem{c13}
Aaron, D.A.,\textit{ et al.}: Control barrier function based quadratic programs for safety critical systems. IEEE Trans. Autom. Control, vol. 62, pp. 3861–3876(2017) 10.1109/TAC.2016.2638961.
\bibitem{c14}
Li, W.,\textit{ et al.}: Safety Barrier Certificates for Collisions-Free Multirobot Systems. IEEE Trans. Robot., vol. 33, pp. 661–674(2017) 10.1109/TRO.2017.2659727.
\bibitem{c15}
Aaron, D.A.,\textit{ et al.}: Control barrier functions: Theory and applications. In Proc. 2019 18th Eur. Control. Conf. (ECC), pp. 3420–3431, Italy(2019) 10.23919/ECC.2019.8796030.
\bibitem{c16}
Urs, B.,\textit{ et al.}: Control barrier certificates for safe swarm behavior. IFAC-PapersOnLine, vol. 48, pp. 68–73(2015) 10.1016/j.ifacol.2015.11.154.
\bibitem{c17}
Mahato, E.,\textit{ et al.}: Collision-free formation control for quadrotor networks based on distributed quadratic programs. in Proc. 2019 Am. Control Conf. (ACC), pp. 3335–3340, USA(2019) 10.23919/ACC.2019.8814603.
\bibitem{c18}
Tatsuya, I.,\textit{ et al.}: Optimization-based distributed flocking control for multiple rigid bodies. IEEE Robot. Autom. Lett., vol. 5, pp. 1891–1898(2020) 10.1109/LRA.2020.2969950.
\bibitem{c19}
Yuki, O.,\textit{ et al.}: Control input design for a robot swarm maintaining safety distances in crowded environment. Symmetry, vol. 13, pp. 478(2021) 10.3390/sym13030478.
\bibitem{c20}
Chuan, S.,\textit{ et al.}: An active safety control method of collision avoidance for intelligent connected vehicle based on driving risk perception. J. Intell. Manuf., vol. 32, pp. 1249-1269(2021) 10.1007/s10845-020-01605-x
\bibitem{c21}
Quan N., Koushil S.: Exponential control barrier functions for enforcing high relative-degree safety-critical constraints. in Proc. 2016 Am. Control Conf., pp. 322–328, USA(2016) 10.1109/ACC.2016.7524935.
\end{thebibliography}

\end{document}